\documentclass[twoside,11pt]{article}

\usepackage{blindtext}
\usepackage{amsmath}
\usepackage{tcolorbox}
\usepackage{algorithm}
\usepackage{bm}
\usepackage{algorithmic}
\usepackage{amsthm} 
\usepackage{dsfont}
\usepackage{amsfonts} 
\newtheorem{assumption}{Assumption}
\usepackage{natbib}
\usepackage{graphicx}
\usepackage{natbib}
\usepackage{booktabs}
\usepackage{float}
\usepackage{multirow}
\usepackage{subcaption}
\tcbuselibrary{skins}
\tcbset{
    generator_box/.style={
        enhanced,
        colback=yellow!10,
        colframe=yellow!50!black,
        fonttitle=\bfseries,
        title=Generator Prompt,
        attach boxed title to top left={yshift=-\tcboxedtitleheight/2},
        boxed title style={size=small,colback=yellow!50!black,colframe=yellow!50!black},
        coltitle=white,
    },
    discriminator_box/.style={
        enhanced,
        colback=blue!10,
        colframe=blue!50!black,
        fonttitle=\bfseries,
        title=Discriminator Prompt,
        attach boxed title to top left={yshift=-\tcboxedtitleheight/2},
        boxed title style={size=small,colback=blue!50!black,colframe=blue!50!black},
        coltitle=white,
    }
}
\usepackage[preprint]{jmlr2e}
\usepackage{lastpage}
\begin{document}

\title{Uncertainty-Aware Multimodal Learning via Conformal Shapley Intervals}

\author{\name Mathew Chandy \email mathewchandy@g.ucla.edu \\
       \addr Department of Statistics and Data Science, UCLA
       \AND
       \name Michael Johnson \email michael.johnson6@merck.com \\
       \addr Merck \& Co. 
       \AND
       \name Judong Shen \email judong.shen@merck.com\\
     \addr Merck \& Co.
       \AND
       \name Devan V. Mehrotra \email  devan\_mehrotra@merck.com \\
       \addr 
       Merck \& Co.
       \AND
       \name Hua Zhou \email huazhou@ucla.edu \\
       \addr Department of Biostatistics, UCLA
       \AND
       \name Jin Zhou \email jinjinzhou@ucla.edu \\
       \addr Department of Biostatistics, UCLA
       \AND
       \name Xiaowu Dai \email  daix@ucla.edu \\
       \addr Departments of Statistics and Data Science, and of Biostatistics, UCLA}
       
\maketitle     
%\editor{My editor}

\begin{abstract}\label{sec:abstract}
Multimodal learning combines information from multiple data modalities to improve predictive performance. However, modalities often contribute unequally and in a data dependent way, making it unclear which data modalities are genuinely informative and to what extent their contributions can be trusted. Quantifying modality level importance together with uncertainty is therefore central to interpretable and reliable multimodal learning.  We introduce \emph{conformal Shapley intervals}, a framework that combines Shapley values with conformal inference to construct uncertainty-aware importance intervals for each modality. Building on these intervals, we propose a modality selection procedure with a provable optimality guarantee: conditional on the observed features, the selected subset of modalities achieves performance close to that of the optimal subset. We demonstrate the effectiveness of our approach on multiple datasets, showing that it provides meaningful uncertainty quantification and strong predictive performance while relying on only a small number of informative modalities.
\end{abstract}

\section{Introduction}
\label{sec:intro}
Multimodal machine learning integrates information from multiple heterogeneous data modalities and has become a core paradigm in modern artificial intelligence. It underpins a wide range of systems, including audio–visual speech recognition \citep{kim2025multitask, kim2025mohave}, vision–language modeling \citep{luo2025vision, sun2025ds, zhong2025time}, and multimodal large language models \citep{zhang2025mm, lian2025affect, song2025modularized}. Beyond empirical success, recent theoretical work shows that combining modalities can yield fundamentally stronger generalization guarantees than unimodal learning \citep{dai2023orthogonalized, lu2023theory}.

At the same time, the availability of multiple modalities raises a central inferential question: \emph{how much does each modality contribute to predictive performance, and how do these contributions depend on the input?} In many multimodal settings, modalities exhibit substantial redundancy or context dependent relevance, so that their marginal contributions are neither uniform nor transparent \citep{peng2025modalities, chaudhuri2025closer, ma2025improving}. Treating all modalities as equally informative can obscure the statistical role of individual data modalities and complicate principled interpretation of multimodal predictors. Consequently, quantifying modality-level importance and identifying informative subsets of modalities are essential for interpretable and reliable multimodal learning.

Shapley values, originally introduced in cooperative game theory \citep{shapley1951notes}, provide a principled framework for modality attribution by assigning each modality a contribution equal to its expected marginal gain when added to a randomly selected coalition. In multimodal learning, each modality is treated as a player, and its Shapley value measures the average change in predictive performance induced by including that modality across all subsets of modalities. Shapley values have therefore been widely adopted in machine learning for feature attribution \citep{lundberg2017unified, alkhatib2025prediction, chowdhury2025rankshap} and modality-level model selection \citep{he2024efficient}. In practice, however, Shapley values are typically computed as point estimates based on a finite training sample and a learned predictive model, and thus inherit uncertainty from both training data and model estimation. This uncertainty can lead to non-negligible variability in estimated modality contributions, particularly in multimodal settings where modalities are highly correlated or redundant. Without explicit uncertainty quantification, differences in estimated modality importance may not be statistically meaningful and may vary substantially across inputs.

In this work, we argue that modality importance should be treated as a statistical object and equipped with finite-sample uncertainty guarantees. We introduce \emph{conformal Shapley intervals}, a framework that integrates Shapley value attribution with conditional conformal inference to construct valid, input-dependent uncertainty intervals for modality importance. Our approach provides uncertainty guarantees that are conditional on the observed features, allowing modality relevance to vary across inputs \citep{gibbs2025conformal}. This yields a strictly stronger form of uncertainty quantification than existing inference methods for Shapley values \citep{watson2023explaining}, which provide only marginal coverage guarantees and do not condition on covariates.

Building on these uncertainty intervals, we develop a modality selection procedure with a provable optimality guarantee. Specifically, conditional on the observed features, the selected subset of modalities achieves predictive performance close to that of the optimal subset. This establishes a direct connection between uncertainty-aware attribution and principled decision making in multimodal learning. We validate our approach on synthetic regression data, image classification with MNIST, and a real Alzheimer’s disease dataset, demonstrating that conformal Shapley intervals provide valid uncertainty quantification and strong predictive performance while frequently relying on only a small number of informative modalities.

\section{Background}
\label{sec:preliminaries}

\subsection{Traditional Shapley Values}
\label{subsec:shapley}

Consider a cooperative game with players indexed by $[p]=\{1,\dots,p\}$. 
For any coalition $S \subseteq [p]$, let $\mathrm{val}(S)\in\mathbb R$ denote the value achieved by coalition $S$. 
The Shapley value \citep{shapley1951notes} of player $j\in[p]$ is defined as 
\begin{equation}
\label{eqn:tradshapley}
\varphi_j
=
\sum_{S \subseteq [p]\setminus\{j\}}
\frac{|S|!(p-|S|-1)!}{p!}
\bigl[
\mathrm{val}(S\cup\{j\})-\mathrm{val}(S)
\bigr].
\end{equation}
The term $\frac{|S|!(p-|S|-1)!}{p!}$  is a weight based on the probability of player $j$ being added to a certain coalition $S$. 
This weighted average measures the average contribution
of player j to a coalition which does not include player $j$.

In multimodal learning, each \emph{modality} is treated as a player and the value function is typically defined in terms of predictive performance. Shapley values therefore provide a principled approach to modality attribution, accounting for interactions among modalities.

\subsection{Multimodal Learning}
\label{subsec:multimodal}
We consider a supervised multimodal learning setting with $p$ modalities. Let
$\mathcal{X} = \mathcal{X}_1 \times \cdots \times \mathcal{X}_p$
denote the modality space, where $\mathcal{X}_j$ corresponds to modality $j$, and let $Y \in \mathcal{Y}$ be the response variable. We observe i.i.d. data
$\{(x_i, y_i)\}_{i=1}^n$, where $x_i = (x_{i1}, \dots, x_{ip}) \in \mathcal{X}$.
For any subset of modalities $S \subseteq [p]$, we write $x_{i,S} = \{x_{ij} : j \in S\}$ for the restriction of the input to modalities in $S$. 

Let $\mathcal{A}$ denote a learning algorithm that maps a training dataset restricted to modalities $S \subseteq [p]$ to a predictor
$\hat\mu_S = \mathcal{A}\bigl(\{(x_{i,S}, y_i)\}_{i \in \mathcal I_1}\bigr) \in \mathcal F$,
where $\mathcal I_1 \subseteq \{1,\dots,n\}$ indexes the training data and $\mathcal F$ denotes the function class of predictors. Given a loss function $\ell : \mathcal{Y} \times \mathcal{Y} \to \mathbb{R}_+$, the predictive performance of $\hat\mu_S$ at an observation $(x_i,y_i)$ is evaluated by $\ell(y_i, \hat\mu_S(x_{i,S}))$.

\begin{figure}
    \centering
    \includegraphics[width=0.75\linewidth]{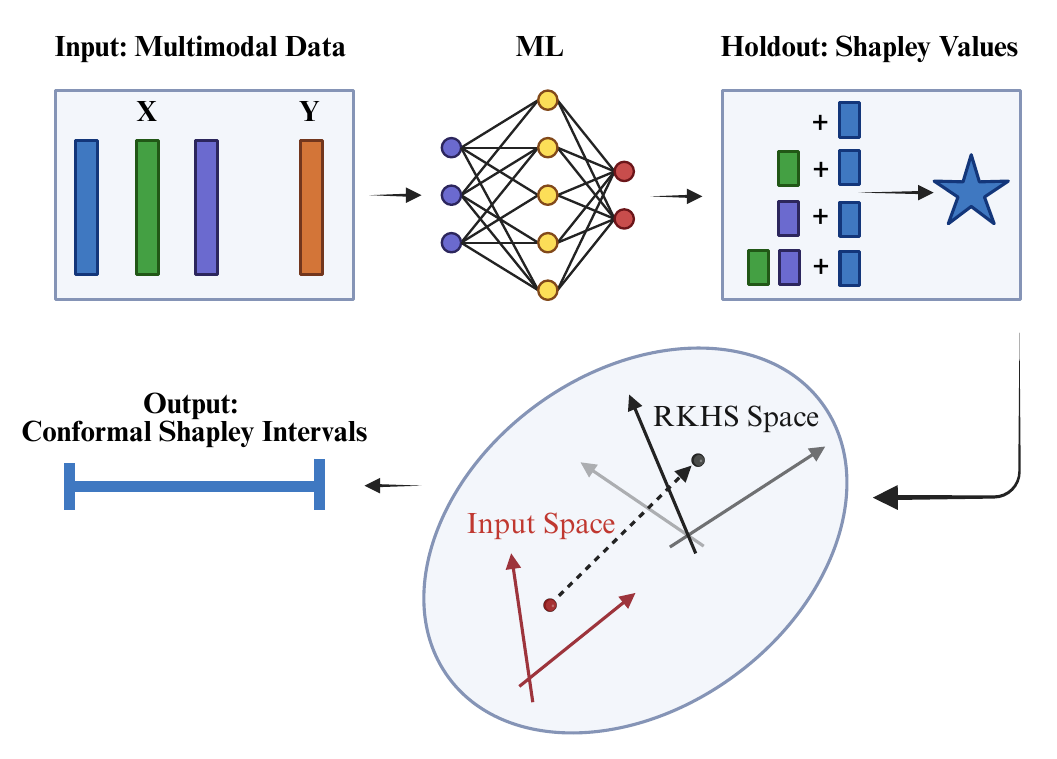}
    \caption{Diagram of conformal Shapley intervals in multimodal learning.}
    \label{fig:shap_mml_diagram}
\end{figure}

\subsection{Conformal Inference}
\label{subsec:conformal}
Conformal inference provides distribution-free, finite-sample guarantees for uncertainty quantification without requiring assumptions on the data distribution \citep{shafer2008tutorial}. 
A common implementation is split conformal inference \citep{lei2018distribution}. The data are randomly partitioned into a training set $\mathcal I_1$ and a calibration set $\mathcal I_2$. A predictor $\hat \mu$ is trained on $\mathcal I_1$, and conformity scores are computed on $\mathcal I_2$. For regression, a standard conformity score is
$s_i = \bigl| y_i - \hat \mu(x_i) \bigr|$, for $i \in \mathcal I_2$.
Prediction intervals are then formed by calibrating the empirical quantile of $\{s_i\}_{i \in \mathcal I_2}$. Given exchangeable data, conformal methods construct prediction sets with guaranteed finite-sample coverage under minimal conditions. 
While split conformal inference guarantees marginal coverage, the resulting prediction intervals do not adapt to the observed features. 

However, in many applications such as multimodal learning, uncertainty is inherently input dependent, and marginal guarantees can be overly conservative or misleading for individual inputs. Conditional conformal inference \citep{gibbs2025conformal} addresses this limitation by estimating feature-dependent quantiles of the conformity scores, yielding prediction sets whose coverage adapts to the observed covariates. This produces uncertainty guarantees that are conditional on the input.
In this work, we extend conditional conformal inference to modality-level Shapley values, and to leverage the confidence intervals to guide principled modality selection with theoretical guarantees.

\section{Conformal Shapley Intervals}
\label{sec:method}
\subsection{Prediction-Based Value Functions}
\label{sec:value}
We define the value of a modality subset $S \subseteq [p]$ at an observation $(x_i, y_i)$ as
\begin{equation}
\label{eqn:defval}
\mathrm{val}(x_i, y_i; S)
=
\ell(y_i, \hat\mu_\emptyset)
-
\ell(y_i, \hat\mu_S(x_{i,S})),
\end{equation}
where $\hat\mu_\emptyset$ denotes a baseline predictor that does not use any modalities, and $\hat\mu_S$ denotes a predictor trained using modalities in $S$. This value function measures the reduction in prediction loss achieved by including modalities in $S$ relative to the baseline.

To understand the population-level behavior of the value function in \eqref{eqn:defval}, we consider its expectation $\mathbb{E}[\mathrm{val}(X,Y;S)]$. Suppose that both $\hat\mu_\emptyset$ and $\hat\mu_S$ are optimal predictors within their respective function classes, 
$\hat\mu_\emptyset = \arg\min_{c \in \mathcal{Y}} \mathbb{E}[\ell(Y,c)]$ and $\hat\mu_S = \arg\min_{\mu_S \in \mathcal{F}} \mathbb{E}[\ell(Y,\mu_S(X_S))]$.
Then
\begin{equation}
\label{eqn:utilityofval}
\begin{aligned}
& \mathbb{E}[\mathrm{val}(X,Y;S)]=
\mathbb{E}[\ell(Y,\hat\mu_\emptyset)]
-
\mathbb{E}[\ell(Y,\hat\mu_S(X_S))] \\
&\quad\quad =
\inf_{c \in \mathcal{Y}} \mathbb{E}[\ell(Y,c)]
-
\inf_{\mu_S \in \mathcal{F}} \mathbb{E}[\ell(Y,\mu_S(X_S))].
\end{aligned}
\end{equation}
Note that the expected value in \eqref{eqn:utilityofval} coincides with the utility function for multimodal learning studied by \citet{he2024efficient}. This utility is monotone non-decreasing in the modality set 
$S$, implying that adding modalities cannot decrease expected predictive performance.

Moreover, under commonly used loss functions, the utility in \eqref{eqn:utilityofval} admits additional structure. In classification with cross-entropy loss, it coincides with the mutual information between the modality subset $S$ and the target $Y$. In regression with quadratic loss, \eqref{eqn:utilityofval} reduces to the variance of the conditional expectation, $\mathrm{Var}(\mathbb{E}[Y \mid S])$. These characterizations motivate the use of the value function in \eqref{eqn:defval} as the basis for Shapley-value--based modality importance.

\subsection{Construction of Conformal Shapley Intervals}
\label{sec:conditional}

Fix a modality $j \in [p]$. For each calibration point $i \in \mathcal I_2$, we compute the instance-level Shapley value
\begin{equation}
\label{eqn:definsshapley}
\begin{aligned}
\varphi_j(x_i)
=
& \sum_{S \subseteq [p]\setminus\{j\}}
 \frac{|S|!(p-|S|-1)!}{p!} \cdot
\Bigl[
\mathrm{val}(x_i, y_i; S \cup \{j\})
-
\mathrm{val}(x_i, y_i; S)
\Bigr].
\end{aligned}
\end{equation}
The value function in \eqref{eqn:definsshapley} is defined by \eqref{eqn:defval}.
Rather than treating $\varphi_j$ as a fixed quantity, as in the classical Shapley value \eqref{eqn:tradshapley}, we view $\varphi_j(X)$ in \eqref{eqn:definsshapley} as a random variable whose distribution depends on the input $X$. Our goal is to construct uncertainty intervals for $\varphi_j(x_{n+1})$ with finite-sample coverage conditional on the observed features $X=x_{n+1}$.

To estimate conditional quantiles of $\varphi_j(X)$, we use reproducing kernel Hilbert spaces (RKHS), which offer flexibility and computational tractability \citep{wahba1990spline, dai2023kernel}. Let $K:\mathcal X\times\mathcal X\to\mathbb R$ be a positive definite kernel with associated RKHS $\mathcal H_K$. To allow for additional finite-dimensional structure, let $\Omega:\mathcal X\to\mathbb R^d$ denote a feature map orthogonal to $\mathcal H_K$. We consider the function class
$\mathcal H
=
\bigl\{
h(\cdot)=h_K(\cdot)+\Omega(\cdot)^\top\beta
:
h_K\in\mathcal H_K,\;
\beta\in\mathbb R^d
\bigr\}$.
For a quantile level $\tau\in(0,1)$, we estimate the conditional $\tau$-quantile function $h^\tau_{\varphi_j}\in\mathcal H$ by solving the regularized empirical risk minimization problem
\begin{equation}
\label{eqn:gibbsmin}
\begin{aligned}
\hat h^{\tau}_{\varphi_j}
 =
\arg\min_{h \in \mathcal{H}}
\Bigg\{
 \frac{1}{|\mathcal I_2|}
& \sum_{i \in \mathcal I_2\cup \{n+1\}}
\ell_{\tau}\bigl(h(x_i), \varphi_j(x_i)\bigr) + 
\lambda_1 \|h_K\|_{K}^2 + \lambda_2||\beta||^2_2
\Bigg\},
\end{aligned}
\end{equation}
where $\ell_\tau$ denotes the pinball loss,
$\ell_\tau(u,v)
=
(\tau-\mathbf 1\{v\le u\})(v-u)$,
and $\lambda_1,\lambda_2>0$ are regularization parameters. Following \citet{gibbs2025conformal}, the inclusion of a symbolic test point $(x_{n+1},\varphi_j(x_{n+1}))$ preserves exchangeability and enables conditional conformal guarantees.

\begin{algorithm}[t]
    \begin{algorithmic}
    \caption{Conformal Shapley Intervals for Modalities}\label{alg:conshapley}
        \STATE {\bfseries Input:} Data $(x_i, y_i), i=1,...,n$, new feature $x_{n+1}$, miscoverage level $\alpha \in (0,1)$, learning algorithm $\mathcal{A}$.
        \STATE {\bfseries Output:} Conformal Shapley interval for each modality $j \in [p]$.
        \STATE Randomly split $\{1,...,n\}$ into two equal sized subsets $\mathcal{I}_1, \mathcal{I}_2$ of length $m = n/2$.
        \STATE Compute $\hat \mu_S =\mathcal{A}(\{(x_{i,S},y_i):i\in\mathcal{I}_1\})$, $S \subseteq [p]$.
        \STATE Compute $\varphi_j(x_i), i \in \mathcal I_2, j\in [p]$ by \eqref{eqn:definsshapley}.
        \STATE Estimate $\hat h_{\phi_j}^{\alpha/2q}$ and $\hat h_{\phi_j}^{1-\alpha/2q}$ by \eqref{eqn:gibbsmin}, $j\in[p]$.
        \STATE Compute and output interval $\hat C_j(x_{n+1})$  by \eqref{eqn:defconfshapley} for $j\in[p]$.
    \end{algorithmic}
    
\end{algorithm}

Finally, given a miscoverage level $\alpha\in(0,1)$ and a target upper bound $q$ on the number of selected modalities, we set $\tau=\alpha/(2q)$ and $\tau=1-\alpha/(2q)$ in \eqref{eqn:gibbsmin} to estimate lower and upper conditional quantile functions. We then define the \emph{conformal Shapley interval} for modality $j\in[p]$ as
\begin{equation}
\label{eqn:defconfshapley}
\widehat C_j(x_{n+1})
=
\Bigl[
\hat h^{\alpha/(2q)}_{\varphi_j}(x_{n+1}),
\;
\hat h^{1-\alpha/(2q)}_{\varphi_j}(x_{n+1})
\Bigr].
\end{equation}
We show in Section~\ref{sec:theory} that these intervals provide finite-sample uncertainty quantification for modality-level Shapley values that is conditional on the observed input and form the basis of our uncertainty-aware modality selection procedure. The procedure for computing conformal Shapley intervals is summarized in Algorithm~\ref{alg:conshapley}.

\subsection{Modality Selection and Inference via Conformal Shapley Intervals}
\label{subsec:selection_inference}
First, let $S_q^\star(X) = \arg\max_{S:\,|S|\le q} \mathbb{E}[\mathrm{val}(X,Y;S)]$ denote the optimal subset of at most $q$ modalities in expectation.
Our goal is to construct an input-dependent approximation to $S_q^\star(X)$ using conformal Shapley intervals in \eqref{eqn:defconfshapley}.
For a test input $x_{n+1}$, we rank modalities according to their upper conditional quantile estimates
$\hat h_{\varphi_j}^{\,1-\alpha/(2q)}(x_{n+1})$.
We then select up to $q$ modalities with the largest scores, retaining only those whose estimated upper quantile is positive:
\begin{equation*}
\begin{aligned}
\hat S_q^\alpha(x_{n+1})
=
\Bigl\{ j \in [p] :
& \hat h_{\varphi_j}^{\,1-\alpha/(2q)}(x_{n+1}) \ge
\hat h_{\varphi_{(q)}}^{\,1-\alpha/(2q)}(x_{n+1})
> 0
\Bigr\},
\end{aligned}
\end{equation*}
where $\hat h_{\varphi_{(q)}}^{\,1-\alpha/(2q)}(x_{n+1})$ denotes the $q$-th order statistic among
$\{\hat h_{\varphi_j}^{\,1-\alpha/(2q)}(x_{n+1})\}_{j=1}^p$.
This procedure yields an input-dependent modality set of size at most $q$.

Second, our framework also supports conditional hypothesis testing for modality importance, allowing the computation of $p$-values when conditioning on a subset of the original features or on an auxiliary set of covariates, denoted by $Z$.
Let the conditional median estimator be $\hat h_{\varphi_j}^{0.5}(\cdot)$ by \eqref{eqn:gibbsmin}, which is trained using the calibration covariates $\{z_i\}_{i \in \mathcal I_2}$.
Let $Z'_i \subseteq Z_i$ denote a subset of covariates fixed at a target level $\zeta$.
We construct modified covariates
$Z_i^* = Z_i \setminus Z'_i \cup \zeta,  i \in \mathcal I_2$,
and evaluate $\hat h_{\varphi_j}^{0.5}(z_i^*)$ for each $i \in \mathcal I_2$.
These values form a sample of conditional Shapley value point estimates under the constraint $Z'=\zeta$. We form a null hypothesis of modality importance \citep{lei2018distribution} as
\[
H_0:\;
\mathbb{E}\!\left[\hat h_{\varphi_j}^{0.5}(Z_{n+1}^*) \mid \hat\mu \right] = 0.
\]
We can apply a standard $t$-test.
Let 
\begin{equation*}
    \bar h = \frac{1}{m}\sum_{i \in \mathcal I_2} \hat h_{\varphi_j}^{0.5}(z_i^*), \text{ and } s_h =
\sqrt{\frac{1}{m-1}\sum_{i \in \mathcal I_2}\bigl(\hat h_{\varphi_j}^{0.5}(z_i^*) - \bar h\bigr)^2},
\end{equation*}
and define the test statistic
$t_h = \frac{\bar h}{s_h/\sqrt{m}}$.
Then, the resulting $p$-value is given by $1 - F_{m-1}(t_h)$, where $F_{m-1}$ denotes the cumulative distribution function of the Student-$t$ distribution with $m-1$ degrees of freedom.

\section{Theoretical Guarantees}
\label{sec:theory}
We establish theoretical guarantees for conformal Shapley intervals under two data-generating regimes.
In the first regime, we assume
$\{(X_i,Y_i)\}_{i=1}^{n+1} \overset{\text{i.i.d.}}{\sim} P$,
and write $\mathbb P_P$ and $\mathbb E_P$ for probability and expectation under $P$.
Let $P_X$ denote the marginal distribution of $X$ and $P_{Y|X}$ the conditional distribution of $Y$ given $X$.
This regime corresponds to the standard in-distribution generalization setting.

In the second regime, we assume that the training data satisfy
$\{(X_i,Y_i)\}_{i=1}^{n} \overset{\text{i.i.d.}}{\sim} P$,
while the test covariate follows a reweighted marginal distribution
$X_{n+1} \sim \frac{f(x)}{\mathbb E_P[f(X)]}\, dP_X(x),
Y_{n+1}\mid X_{n+1} \sim P_{Y|X}$,
for a nonnegative weighting function $f$ with $\mathbb E_P[f(X)]>0$.
We write $\mathbb P_f$ and $\mathbb E_f$ for probability and expectation under this regime. This regime models covariate shift at test time, allowing domain adaptation and subgroup analysis while preserving the conditional distribution $P_{Y|X}$.

We now present our coverage result, which adapts arguments from Theorem~3 of \citet{gibbs2025conformal} to modality-level Shapley values.
\begin{lemma}[Conditional Coverage of Conformal Shapley Intervals]
\label{lem:conditional}
Let $\mathcal H$ be any vector space, and assume
that for all $f, g \in \mathcal H$, the partial derivative of
$\mathcal R(g + \epsilon f)$ with respect to $\epsilon$ exists.
If $f$ returns nonnegative values with $\mathbb E_P[f(X)] > 0$,
then the prediction set given by $\hat C_j(x^{n+1})$ satisfies the
lower bound
\begin{multline*}
    \mathbb P_f(\phi^{n+1}_j \in \hat C_j(x^{n+1})) \geq 1 - \alpha - \frac{1}{\mathbb E_P[f
(X)]}\cdot \\ \mathbb E\Big[\frac{d}{d \epsilon} \mathcal R(\hat h^{1-\alpha/2}_{\phi_j}+\epsilon f)\Big|_{\epsilon=0} \ - 
\frac{d}{d \epsilon} \mathcal R(\hat h^{\alpha/2}_{\phi_j}+\epsilon f)\Big|_{\epsilon=0}\Big].
\end{multline*}
If we suppose that $\{(X^i, Y^i)\}_{i=1}^{n+1} \overset{i.i.d.}\sim P$, then for $f \in \mathcal F$,
the following two-sided bound is also satisfied,
\begin{multline*}
    \mathbb E[f(x^{n+1})(\mathbf 1\{\phi^i_j\in \hat C_j(x^{n+1})\}-(1-\alpha))]=\\-\mathbb E \Big [\frac{d}{d\epsilon}R(\hat h^{1-\alpha/2}_{\phi_j } + \epsilon f) \Big |_{\epsilon=0}  - \frac{d}{d\epsilon}R(\hat h^{\alpha/2}_{\phi_j } + \epsilon f) \Big |_{\epsilon=0} \Big ] + \epsilon_{int},
\end{multline*}
where $\epsilon_{int}$ is an error such
that $|\epsilon_{int}| \leq \mathbb E[|f(x^i)| \mathbf 1\{\phi^i_j= \hat h^{1-\alpha/2}_{\phi_j }(x^i)\}] \ +  
\mathbb E[|f(x^i)| \mathbf 1\{\phi^i_j= \hat h^{\alpha/2}_{\phi_j }(x^i)\}]$.
\end{lemma}

We consider a collection of feature transformations $\Omega(X_i)=\{\Omega_j(X_{ij})\}_{j=1}^p$ 
where each $\Omega_j(\cdot)$ maps the $j$-th modality to a feature representation.
In experiments, we employ a finite-dimensional representation obtained via dimension reduction, using either principal component analysis or singular value decomposition, to ensure computational tractability and stable estimation.

\begin{assumption}[Moment bounds, non-degeneracy, and bounded utility]
\label{assumption:moment}
There exist constants $C_1,C_2,C_3,C_\varphi,C_f,\rho,B>0$  such that, for all $f\in\mathcal H$ and all $\beta\in\mathbb R^p$ with $\|\beta\|_2=1$,
\begin{align*}
\mathbb E\big[\|\Omega(X_i)\|_2^2\big] &\le C_1 d, \\
\mathbb E\big[|f(X_i)|\,\|\Omega(X_i)\|_2^2\big] &\le C_2\,d\,\mathbb E|f(X_i)|, \\
\mathbb E\big[|\Omega(X_i)^\top\beta|^2\big] &\le C_3, \\
\mathbb E\big[|f(X_i)|(\varphi_j(x_i))^2\big] &\le C_\varphi\,\mathbb E|f(X_i)|, \\
\sqrt{\mathbb E|f(X_i)|^2} &\le C_f\,\mathbb E|f(X_i)|, \\
\mathbb E\big[|\Omega(X_i)^\top\beta|\big] &\ge \rho, \\
\mathbb E|\varphi_j(x_i)|^2 &< \infty, \ |\mathbb{E}[\mathrm{val}(X,Y;S)]|\le B.
\end{align*}
\end{assumption}

Assumption~\ref{assumption:moment} imposes mild moment and
non-degeneracy conditions on the feature transformations
and utility function. Such conditions are standard in RKHS
learning and high-dimensional statistics to guarantee concentration and uniform convergence of empirical processes \citep{vershynin2018high,wainwright1945high,gibbs2025conformal}.

To show near-optimality of model selection for classification, we require another assumption.
\begin{assumption}[Approximate independence]
\label{assumption:independence}
There exist constants $\epsilon_1,\epsilon_2 \ge 0$ such that for all
disjoint $S, S' \subseteq [p]$,
$I(X_{i,S};X_{i,S'} \mid Y) \le \epsilon_1$ and 
$I(X_{i,S};X_{i,S'}) \le \epsilon_2$.
\end{assumption}

Assumption \ref{assumption:independence} has been used in the multimodal learning literature  \citep{he2024efficient} and is  weaker than
the strict conditional
independence assumption
\citep{white2012convex, wu2018multimodal, sun2020tcgm}.

If we select the modalities with $q$-highest
$\hat h_{\phi_j}^{1-\alpha/2q} > 0$, we get the following guarantees
with regards to the optimal set $S^*_q = \arg\max_{S:|S|\leq q} \mathbb{E}[\mathrm{val}(X,Y;S)]$

\begin{theorem}[Near-Optimal Utility of Selected Set for Classification]
\label{thm:classification}
Let $S^*_q$ be the optimal set of size less than or equal to $q$.
Let $\hat S^\alpha_q$ be the set of size less than or equal to $q$ selected
by our algorithm for target 
failure probability $\alpha \in (0,1)$, and let
$\ell(y, \hat y) = -\sum_{y^*\in \mathcal{Y}}\mathbf1\{y=y^*\}\log \hat y$.
Under Assumptions~ \ref{assumption:moment}-\ref{assumption:independence},
the following inequality holds with probability
at least $1 - \alpha - 2\lambda_1 \epsilon_1-2\lambda_2 \epsilon_2$:
\begin{align*}
    &\mathbb{E}[\mathrm{val}(X,Y;\hat S^\alpha_q)]\geq 
\mathbb{E}[\mathrm{val}(X,Y;S^*_q)]\\&\quad\quad -\sum_{j\in S^*_q} 
2\left(\kappa\sqrt{\frac B \lambda_1} + C_\Omega(x)\sqrt\frac{B}{\lambda_2}\right) - 4q \delta.
\end{align*}    
\end{theorem}

To show near-optimality of model selection for regression, we require the following assumption:
\begin{assumption}[Linear-Gaussian structure]
\label{assumption:linearity}
The transformed features $\Omega(X)$ satisfy
\[
\mathbb E[Y_i \mid \Omega(X_i)] = \alpha + \sum_{j=1}^p \beta_j \Omega_j(X_{ij}),
\
\Omega(X) \sim \mathcal N(\mu, \Sigma).
\]
\end{assumption}
This assumption formalizes a linear signal structure after feature transformation and is standard in theoretical analyses of explainable  machine learning and AI methods \citep{lundberg2017unified, wilming2022scrutinizing}.

\begin{theorem}[Near-Optimal Utility of Selected Set for Regression]
\label{thm:regression}
Let $S^*_q$ be the optimal set of size less than or equal to $q$.
Let $\hat S^\alpha_q$ be the set of size less than or equal to $q$ selected
by our algorithm for target 
failure probability $\alpha \in (0,1)$, and
let $\ell(y, \hat y) = (y-\hat y)^2$
Under Assumptions~ \ref{assumption:moment} and \ref{assumption:linearity},
the following inequality holds with probability
at least $1 - \alpha - 2\lambda_1 \epsilon_1-2\lambda_2 \epsilon_2$:
\begin{align*}\mathbb{E}[\mathrm{val}(X,Y;\hat S^\alpha_q)] &\geq \mathbb{E}[\mathrm{val}(X,Y;S^*_q)] \\&- \sum_{j\in S^*_q} 2\left(\kappa\sqrt{\frac B \lambda_1} + C_\Omega(x)\sqrt\frac{B}{\lambda_2}\right)
\end{align*}    
\end{theorem}

Thus, our method offers a conditional guarantee on near-optimality
of the selected set $\hat S^\alpha_q$ in both classification and
regression settings.
The proofs of our theory is shown in
the Supplementary Material.

\section{Empirical Results}
\label{sec:empirical}
We evaluate conformal Shapley intervals on various datasets to assess two key questions:
(i) whether uncertainty-aware modality selection achieves strong predictive performance with fewer modalities; (ii)  whether the proposed intervals provide meaningful, input-dependent uncertainty quantification for modality importance.
The code for these experiments is available at \url{https://anonymous.4open.science/r/shap_mml-B4D1/}.

\subsection{Synthetic Regression}
\label{subsec:reg}
We first evaluate our method on the synthetic regression benchmark of \citet{he2024efficient}, which allows controlled assessment of modality selection under known dependence structures. The input consists of $p=10$ modalities, each of dimension $d=3$.
To generate correlated multimodal features, we construct two covariance components. Let $B \in \mathbb{R}^{pd \times pd}$ be obtained by sampling entries uniformly from $[-1,1]$, multiplying the matrix by its transpose, and normalizing each row by the sum of absolute values. Let $A \in \mathbb{R}^{pd \times pd}$ be a block-diagonal matrix with $p$ diagonal blocks of size $d \times d$, each constructed using the same procedure as $B$, and zeros elsewhere. We then define the covariance matrix
$\Sigma = (1-\epsilon)A + \epsilon B$,
where $\epsilon \in [0,1]$ controls the degree of cross-modality dependence. We sample $n=1000$ observations $X \sim \mathcal{N}(0,\Sigma)$ and generate responses according to
$Y = X^\top \beta + \alpha + \varepsilon$,
where $\beta \in \mathbb{R}^{pd}$ and $\alpha \in \mathbb{R}$ are sampled uniformly from $[-1,1]$, and $\varepsilon \sim \mathcal{N}(0,1)$. We use 250 samples for training, 250 for calibration, and the remaining $500$ samples for testing.
We apply Algorithm~\ref{alg:conshapley} with ordinary least squares as the learner $\mathcal{A}$ and squared error loss. Models are refit for each modality subset. For the RKHS-based quantile estimation, we select regularization parameters $\lambda_1$ and $\lambda_2$ using five-fold cross-validation.

\begin{figure}[t]
    \centering
    \includegraphics[width=0.75\linewidth]{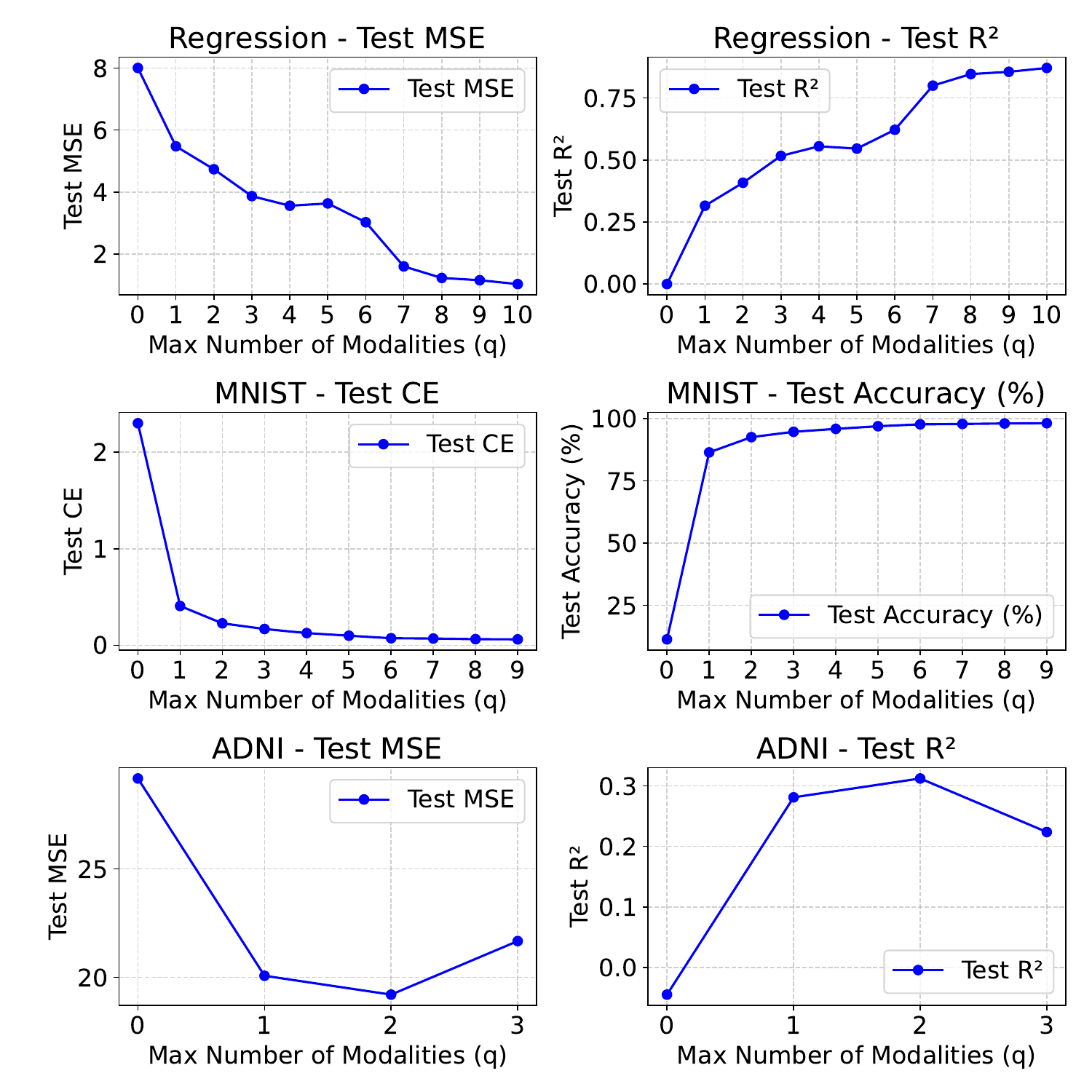}
    \caption{Model selection paths for synthetic regression data set, MNIST, and ADNI. Test MSE and $R^2$ are shown for regression and ADNI data, whereas test cross-entropy (CE) and accuracy (\%) shown for MNIST.}
    \label{fig:facet_model_selection_paths}
\end{figure}

\paragraph{Results.}
The first row of Figure~\ref{fig:facet_model_selection_paths} shows the resulting modality selection paths as a function of the maximum allowed subset size $q$. As $q$ increases, predictive performance, measured by test MSE and $R^2$, improves nearly monotonically and rapidly approaches the optimal model using all modalities. Importantly, near-optimal performance is achieved using substantially fewer modalities, illustrating the effectiveness of uncertainty-aware selection. These results empirically support the conditional near-optimality guarantees established in Section~\ref{sec:theory}.
While we present selection paths that adapt to individual test points, the conformal Shapley framework also admits global selection strategies. For example, one may integrate the upper Shapley quantile function $\hat h^{1-\alpha/(2q)}_{\varphi_j}(x)$ over a target distribution on $\mathcal{X}$ and select a single modality subset that maximizes this integrated score.

% \begin{figure}
%     \centering
%     \includegraphics[width=1\linewidth]{manuscript/figures/modality_attribution_boxplot.pdf}
%     \caption{Distributions of Shapley values for each modality in
%     \citet{he2024efficient}'s regression setting.}
%     \label{fig:boxplot}
% \end{figure}

% \begin{figure}s
%     \centering
%     \includegraphics[width=1\linewidth]{manuscript/figures/shapley_value_intervals.pdf}
%     \caption{80\% split conformal intervals for Shapley values for
%      each modality in \citet{he2024efficient}'s regression setting.}
%     \label{fig:intervals}
% \end{figure}

\subsection{MNIST Classification}
\label{subsec:class}

We next evaluate our method on the MNIST image classification task, using the patch-based multimodal construction. The MNIST dataset consists of 50{,}000 training images and 10{,}000 test images of handwritten digits from 0 to 9, where each image is grayscale with resolution $28 \times 28$ \citep{lecun2010mnist}. In our experiments, each image is partitioned into $p=9$ spatial patches, so that each patch corresponds to a rectangular region of side length 9 or 10 pixels and is treated as a separate modality.
For the learner $\mathcal{A}$, we use a lightweight convolutional neural network that takes a $28 \times 28 \times 1$ grayscale input and consists of a single convolutional layer with 32 filters of size $3 \times 3$ and ReLU activation, followed by $2 \times 2$ max pooling. The resulting feature maps are flattened and passed to a fully connected hidden layer with 128 ReLU-activated units. The output layer is a 10-way softmax that produces class probabilities. Models are trained using the Adam  with learning rate $10^{-3}$ and the sparse categorical cross-entropy loss. To control computational cost when refitting models for different modality subsets, each network is trained for a single epoch with mini-batches of size 16 and a 0.1 validation split.

\begin{figure}[t]
    \centering
    \includegraphics[width=0.75\linewidth]{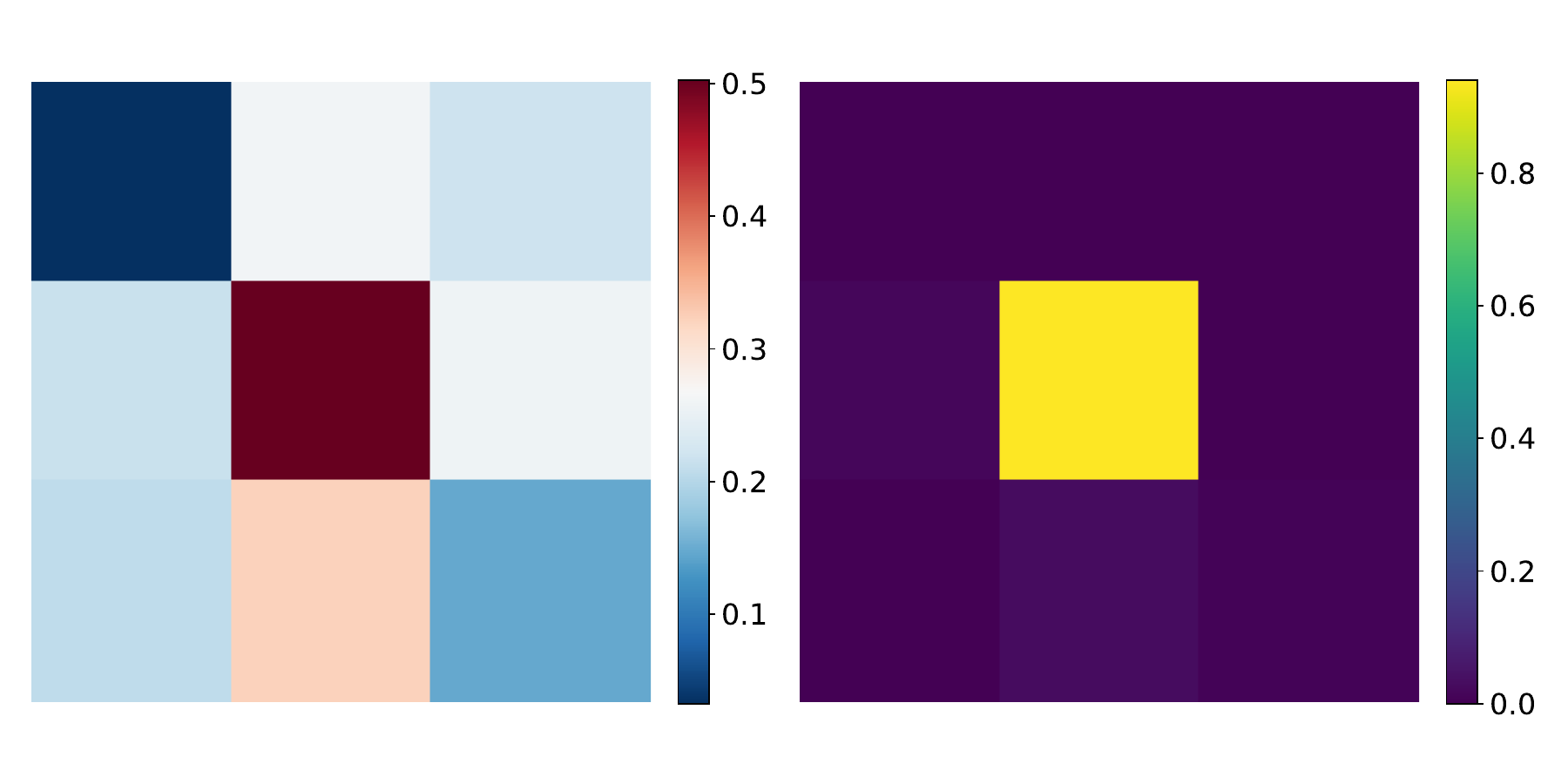}
    \caption{Left panel: Mean Shapley values $\hat\varphi_j$ for each image patch $j$, averaged over the calibration set.  Right panel: Selection frequency of each patch when at most one patch is selected ($q=1$).}
    \label{fig:mnist}
\end{figure}

\paragraph{Results.}
The left panel of Figure~\ref{fig:mnist} shows the mean Shapley values $\hat\varphi_j$ for each patch, averaged over the calibration set. As expected, the central patch exhibits the highest average importance, while corner patches have lower marginal contributions, showing the spatial concentration of digit information. However, these marginal summaries do not capture input-dependent variation in modality relevance.

We apply Algorithm~\ref{alg:conshapley} with log loss to construct conformal Shapley intervals and perform uncertainty-aware modality selection. In the RKHS quantile estimation step, we apply principal component analysis for dimension reduction and set $\lambda_1=\lambda_2=0.01$. The second row of Figure~\ref{fig:facet_model_selection_paths} shows test accuracy as a function of the maximum number of selected patches $q$. Predictive performance increases monotonically with $q$ and remains strong even when $q=1$, indicating that near-optimal classification accuracy can be achieved using a single, adaptively selected patch.

The right panel of Figure~\ref{fig:mnist} shows the selection frequency of each patch when $q=1$. As expected, the central patch is selected most frequently. However, unlike marginal or point-estimate-based selection rules, our conditional method also selects non-central patches for a non-negligible fraction of test images. This result shows that conformal Shapley values capture input-dependent modality relevance, allowing different patches to be selected for different inputs.

\begin{table*}[t]
\centering
\caption{Global modality-importance $P$-values by subgroup.}
\label{tab:global-p-values}
\footnotesize
\resizebox{\textwidth}{!}{%
\begin{tabular}{lccccc ccc cccc}
\toprule
& \multicolumn{2}{c}{(a) Sex} 
& \multicolumn{3}{c}{(b) APOE} 
& \multicolumn{3}{c}{(c) Diagnosis} 
& \multicolumn{4}{c}{(d) Education} \\
\cmidrule(lr){2-3} \cmidrule(lr){4-6} \cmidrule(lr){7-9} \cmidrule(lr){10-13}
Modality
& Female & Male
& APOE2 & APOE3 & APOE4
& AD & MCI & CU
& Postgrad & College & SomeCollege & $\le$HS \\
\midrule
A$\beta$
& 0.63 & 0.30
& 1.00 & 8.9e{-15} & 1.00
& 0 & 1.00 & 1.8e{-7} 
& 0.035 & 1.5e{-7} & 1.00 & 1.00 \\

Tau
& 0.014 & 1.00
& 0.94 & 0.44 & 0.41
& 1.00 & 1.00 & 0.008 
& 9.9e{-5} & 1.00 & 2.6e{-5} & 1.00 \\

MRI
& 2.1e{-4} & 1.00
& 9.2e{-5} & 0 & 1.00
& 5.2e{-10} & 1.00 & 0.035
& 0.64 & 0.15 & 3.0e{-8} & 1.00 \\
\bottomrule
\end{tabular}}
\end{table*}

\subsection{Alzheimer’s Disease Data (ADNI)}
\label{subsec:adni}

We finally apply our method to a real-world biomedical dataset from the Alzheimer’s Disease Neuroimaging Initiative \citep[ADNI,][]{jack2008alzheimer}, where modality selection is both scientifically meaningful and practically important. The response variable is the Preclinical Alzheimer’s Cognitive Composite \citep{donohue2014preclinical}, a continuous score measuring cognition and memory. %that ranges from $-35.07$ to $6.58$ in the dataset. 
We consider three neuroimaging modalities that capture different stages of Alzheimer’s disease progression \citep{jack2010hypothetical}. The first modality is amyloid-$\beta$, a protein fragment that accumulates early in the disease process, often decades before clinical symptoms. The second modality is tau, a protein associated with neurofibrillary tangles that typically appear later. The third modality is cortical thickness derived from MRI scans, which measures neurodegeneration through structural brain atrophy. understanding which modalities are most informative for predicting cognitive and memory decline is critical for translating predictive insights into actionable strategies, such as prioritizing biological pathways and targets for therapeutic development \citep{leng2021neuroinflammation}.

In addition to imaging data, we include five demographic and genetic covariates: sex, education level, APOE genotype, diagnosis status, and age. These covariates are used to condition the conformal Shapley intervals. The dataset contains 557 subjects, where we use 350 samples for training and calibration (split evenly) and reserve the remainder for testing. For the learner $\mathcal A$, we use XGBoost with 200 estimators, learning rate 0.1, and maximum tree depth 3. Regularization parameters for the RKHS quantile estimator are selected via five-fold cross-validation for each choice of the maximum number of modalities $q$.

%We also include five other covariates: education ($\leq$ High School, Some College, College Graduate, Post Graduate), and age. 
%In the literature, we have seen that lower Education is associated with lower risk of Alzheimer’s but that is not to say there is necessarily a causal  effect \citep{walters2025educational}.  In this case, the conformal Shapley intervals are conditioned on the five covariates listed, whereas the predictive model does not include this feature information. The original sample size is 700, but after removing rows with missing values, we are left with 557.

\paragraph{Results on uncertainty-aware modality selection.}
The third row of Figure~\ref{fig:facet_model_selection_paths} reports test MSE and $R^2$ as functions of the maximum number of selected modalities $q$. Performance improves rapidly for  for $q \leq 2$. Performance may plateau or slightly decrease for larger $q$, reflecting the inclusion of less informative or noisy modalities when no explicit penalty on model size is imposed. Because modality selection is conditional on patient covariates, the selected modality subset may vary across individuals. We also observe that using higher conditional quantiles yields smoother and more stable selection paths, whereas lower quantiles introduce greater variability.

\paragraph{Results on modality importance by subgroups.}
Beyond individual-level selection, our framework enables global hypothesis testing for modality importance conditional on demographic and genetic covariates, where the $p$-values are shown in Table~\ref{tab:global-p-values}.

Table~\ref{tab:global-p-values}(a) shows that, conditional on sex, tau and MRI are significant for female patients, while amyloid-$\beta$ is not. In contrast, for male patients, amyloid-$\beta$ is closest to being significant, whereas tau and MRI have $p$-values of 1. These patterns align with evidence that Alzheimer’s disease progression differ by sex \citep{mei2025unraveling,arenaza2024sex, filon2016gender, moutinho2025women}.

Table~\ref{tab:global-p-values}(b) shows that modality importance, when conditioned on APOE genotype, exhibits clinically interpretable patterns. For APOE2 carriers, who are at lower genetic risk, early biomarkers such as amyloid-$\beta$ and tau are not significant, whereas MRI-based measures are informative. This is consistent with evidence that APOE2 confers relative resilience to amyloid and tau pathology, but not necessarily to neurodegeneration \citep{grothe2017multimodal}. For APOE3 carriers, amyloid-$\beta$ and MRI are highly significant, while tau is not, reflecting the limited prognostic value of APOE3 alone and its relative resilience to tau-related pathology \citep{liu2024updates}. In contrast, for APOE4 carriers, none of the imaging modalities are significant, suggesting that genetic risk already accounts for a substantial portion of the variability in early cognitive decline \citep{mares2025apoe}.

Table~\ref{tab:global-p-values}(c) reports results stratified by diagnosis status. For patients with Alzheimer’s disease (AD), amyloid-$\beta$ and MRI are significant, consistent with evidence that tau pathology and neurodegeneration are closely coupled in established AD \citep{bejanin2017tau}. For patients with mild cognitive impairment (MCI), none of the modalities are significant, suggesting that modality-level signals are less predictive in this intermediate stage, where cognitive impairment may arise from heterogeneous and non–Alzheimer’s-related causes. For cognitively unimpaired (CU) individuals, amyloid-$\beta$ and tau are significant, while MRI is not, aligning with the Alzheimer’s pathological model in which molecular pathology precedes neurodegeneration \citep{jack2013tracking}.

Table~\ref{tab:global-p-values}(d) summarizes modality importance conditional on education level. For individuals with postgraduate education, tau is significant, while amyloid-$\beta$ and MRI are not. Among those with a college degree, only amyloid-$\beta$ is significant. This pattern is consistent with the brain reserve hypothesis, which suggests that higher educational attainment may mitigate the observable impact of neurodegeneration on cognitive performance \citep{garibotto2008education}, and with prior evidence linking education to reduced amyloid plaque burden in certain populations \citep{honig2024association}. For individuals with some college education, tau and MRI are significant, while amyloid-$\beta$ is not, suggesting that downstream pathology and neurodegeneration play a larger role in this group. For individuals with a high school education or less, none of the modalities are significant, consistent with prior findings that cognitive outcomes exhibit substantial heterogeneity even after accounting for biomarker information, particularly across education levels \citep{roe2011cerebrospinal}.

These experiments show that conformal Shapley framework provides uncertainty-aware, conditional modality attribution that reveals clinically meaningful heterogeneity.

% We sought to achieve
% the best predictive MSE using one modality ($q = 1$). We tried
% regularization parameters $\lambda_1 \in \{1, 0.1, 0.01, 0.001\}$ and
% $\lambda_2 \in \{1, 0.1, 0.01, 0.001\}$. We find
% that for all our hyperparameter choices, our method
% results in lower MSE than simply predicting modality 1.

\section{Related Work}
\label{sec:related}
Our work connects three lines of research: Shapley-based modality attribution, modality selection in multimodal learning, and uncertainty quantification with finite-sample guarantees. To the best of our knowledge, it is the first framework to provide finite-sample, input-dependent uncertainty intervals for modality-level Shapley values and to leverage these intervals to obtain conditional near-optimality guarantees for modality selection. We briefly review the most closely related work in these areas below.

\paragraph{Shapley values and feature attribution.}
Shapley values were originally introduced in cooperative game theory \citep{shapley1951notes} and later adapted to machine learning for feature attribution by \citet{lundberg2017unified}. Subsequent work has focused on efficient computation and scalable approximations of Shapley values in high-dimensional settings \citep{sundararajan2020many, chen2023algorithms, watson2023explaining}. The majority of this literature considers feature-level explanations and treats Shapley values as fixed quantities derived from a learned model. In contrast, our paper explicitly accounts for uncertainty due to finite samples, model estimation, and data variability.

\paragraph{Multimodal learning and modality selection.}
Multimodal learning has been shown to offer both empirical and theoretical advantages over unimodal approaches \citep{yuhas2002integration, lu2023theory}. While incorporating multiple modalities often improves predictive performance, it can also increase computational burden and complicate interpretation. \citet{he2024efficient} proposed a Shapley-based framework for modality selection with near-optimality guarantees under a utility-based formulation. However, that approach treats modality importance as a point estimate and does not quantify uncertainty, relying on marginal Shapley values that do not adapt modality selection to individual inputs or observed covariates. In contrast, our work provides uncertainty-aware, input-dependent modality selection.

\paragraph{Uncertainty quantification for explanations.}
Recent work has begun to address uncertainty in Shapley-based explanations. Bayesian approaches characterize posterior variability of Shapley values under probabilistic modeling assumptions \citep{agussurja2022convergence}, while resampling-based methods estimate variability through the bootstrap \citep{huang2023increasing}.
\citet{watson2023explaining} introduced split conformal inference for feature-level Shapley values, providing finite-sample uncertainty guarantees with marginal coverage. More broadly, conformal inference offers distribution-free guarantees for uncertainty quantification \citep{lei2018distribution}, with classical split conformal methods ensuring marginal coverage and recent advances targeting conditional or locally adaptive guarantees \citep{romano2019conformalized, chernozhukov2021distributional}. In particular, \citet{gibbs2025conformal} proposed a general framework for conditional conformal inference using weighted exchangeability and RKHS-based quantile estimation. Building on this line of work, we apply conditional conformal inference to Shapley values themselves, rather than to prediction errors or outcomes, and provide uncertainty guarantees that are conditional on the observed covariates. This enables uncertainty-aware, input-dependent modality attribution and decision making.

\section{Conclusion}
\label{sec:concluding}

We introduced \emph{conformal Shapley intervals}, a framework for uncertainty-aware modality attribution and selection in multimodal learning. By combining Shapley values with conditional conformal inference, our method provides finite-sample, input-dependent uncertainty guarantees for modality importance and yields conditional near-optimality guarantees for predictive performance.
Our approach moves beyond point-estimate explanations by explicitly quantifying uncertainty in modality relevance and allowing modality selection to adapt to observed covariates. This unifies interpretability, uncertainty quantification, and model selection in a single framework for both regression and classification.
Experiments demonstrate that our method achieves competitive or improved predictive performance while relying on fewer modalities, and yields scientifically interpretable insights in applied settings.

There are several directions for future work. First, exact Shapley value computation can be computationally expensive when the number of modalities is large; incorporating approximation schemes may significantly improve scalability. Second, while our conditional guarantees adapt locally to observed covariates, extending the framework to target structured regions of the covariate space, as in region-adaptive conformal prediction, is a promising direction. Finally, it is of interest to extend our results to sequential or longitudinal settings, where modalities arrive over time, which is especially relevant in biomedical applications.

%%%%%%%%%%%%%%%%%%%%%%%%%%%%%%%%%%%%%%%%%%%%%%%%%

\bibliography{aiAgent}
% \bibliographystyle{apalike}

%%%%%%%%%%%%%%%%%%%%%%%%%%%%%%%%%%%%%%%%%%%%%%%%%%%%%%%%%%%%%%%%%%%%%%%%%%%%%%%
%%%%%%%%%%%%%%%%%%%%%%%%%%%%%%%%%%%%%%%%%%%%%%%%%%%%%%%%%%%%%%%%%%%%%%%%%%%%%%%
% APPENDIX
%%%%%%%%%%%%%%%%%%%%%%%%%%%%%%%%%%%%%%%%%%%%%%%%%%%%%%%%%%%%%%%%%%%%%%%%%%%%%%%
%%%%%%%%%%%%%%%%%%%%%%%%%%%%%%%%%%%%%%%%%%%%%%%%%%%%%%%%%%%%%%%%%%%%%%%%%%%%%%%
\newpage
\appendix
\onecolumn

\section{Results on Marginal Shapley for Alzheimer’s Disease Data (ADNI).}
The results in Section~\ref{subsec:adni} primarily leverage our \emph{conditional} conformal Shapley framework, which enables modality-level inference that adapts to observed feature values. As discussed in Subsection~\ref{sec:conditional}, the procedure involves computing pointwise Shapley values on the calibration set. Aggregating these pointwise estimates yields \emph{marginal} Shapley value distributions, which provide a global summary of modality importance averaged over demographic and genetic covariates.

We report such marginal analyses for the ADNI experiment. Figure~\ref{fig:shapley_density_by_modality} shows the empirical distributions of Shapley values by modality. All modalities exhibit right-skewed distributions, indicating heterogeneous contributions across individuals. Amyloid-$\beta$ has a distribution concentrated near zero, suggesting limited average contribution, while tau and MRI display broader distributions, with MRI exhibiting heavier tails. These marginal distributions reveal heterogeneity that is invisible to single point estimates, but they still average over important sources of variability.

While marginal Shapley analyses enable global inference on modality importance and can inform high-level modeling decisions, they are fundamentally limited: by averaging over covariates, they cannot capture context-dependent modality relevance or support individualized modality selection. In contrast, our conditional conformal approach results in Section~\ref{subsec:adni} explicitly models how modality importance varies with observed features, yielding input-dependent uncertainty intervals and enabling tailored modality selection. This distinction is critical in applications such as Alzheimer’s disease, where modality relevance depends strongly on patient characteristics.

Relying solely on marginal importance in this setting may lead to overly conservative or suboptimal use of costly or invasive imaging modalities, as it cannot distinguish patients for whom a modality is informative from those for whom it is not. By contrast, the conditional framework supports patient-specific decisions, as illustrated by the conditional $p$-values in Table~\ref{tab:global-p-values}. This sharp separation between marginal summary and conditional inference highlights the practical value of our conditional conformal Shapley framework.

\begin{figure}[ht]
\centering
\includegraphics[width=0.5\linewidth]{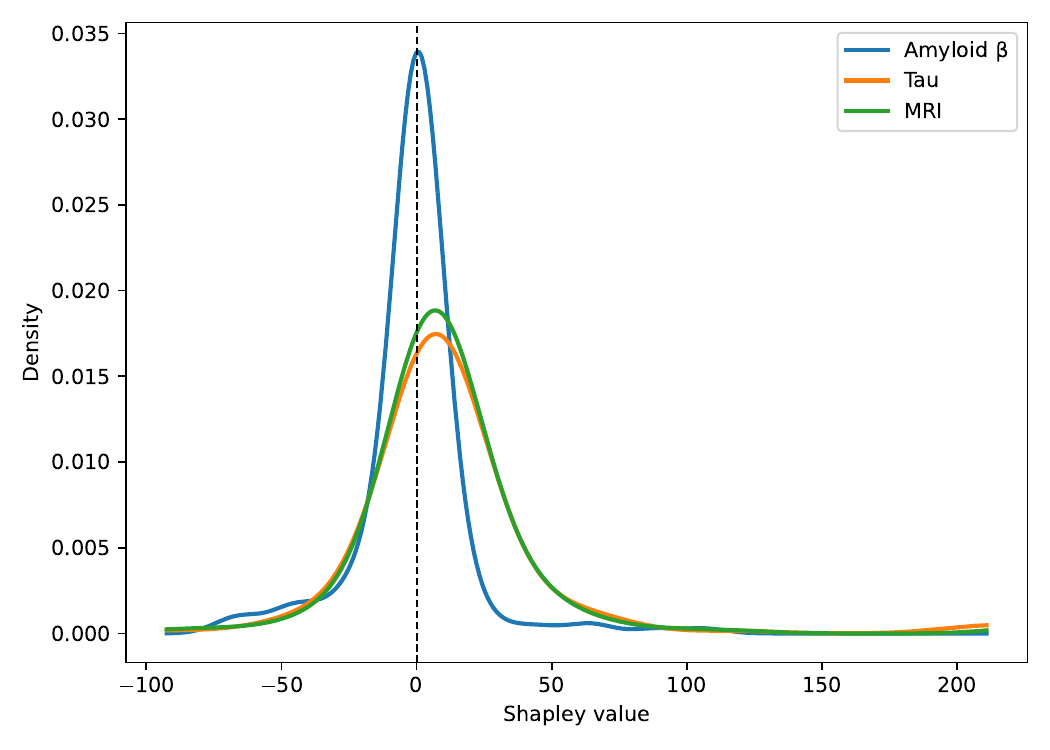}
\caption{Marginal distributions of Shapley values by modality in the ADNI experiment.}
\label{fig:shapley_density_by_modality}
\end{figure}

\section{Theoretical Proofs.}

\begin{proof}[Proof of Lemma~\ref{lem:conditional}]
We first examine the first order conditions of the convex optimization problem for some $\alpha$:
\begin{align*}
\hat h^{\tau}_{\varphi_j}
& =
\arg\min_{h \in \mathcal{H}}
\Bigg\{
\frac{1}{|\mathcal I_2|}
\sum_{i \in \mathcal I_2}
\ell_{\tau}\bigl(h_K(x_i)+\Omega(x_i)^\top \beta, \varphi_j(x_i)\bigr) \\
&+ \lambda_1 \|h_K\|_{K}^2 + \lambda_2||\beta||^2_2
\Bigg\},
\end{align*}
We have that for any fixed $h \in \mathcal{H}$,
$$0 \in \partial_\epsilon \Big (\frac{1}{m+1} \sum_{i=m+1}^{n+1} \ell_\alpha(\hat h^{1-\alpha}_{\varphi_j}(x_i) + \epsilon f(x_i), \varphi_j(x_i))+\mathcal{R}(\hat h^{1-\alpha}_{\varphi_j} + \epsilon f) \Big ) \Big |_{\epsilon = 0}$$

Then the subgradients of the pinball loss term in the direction
of $f(.)$ are:
\begin{align*}
    \left\{\frac{1}{m+1}\left( \sum_{\varphi_j(x_i) \neq \hat h^{1-\alpha}_{\varphi_j }(x_i)} f(x_i)(\alpha - \mathbf 1\{ \varphi_j(x_i) > \hat h^\alpha_{\varphi_j}(x_i)\})\ + \sum_{\varphi_j(x_i)=\hat h^{1-\alpha}_{\varphi_j }(x_i)}t_if(x_i)\right) \Big |t_i \in [\alpha-1,\alpha] \right\}
\end{align*}
Choose $t^{\alpha}_i \in [\alpha-1,\alpha]$ such that they set
the subgradient to 0. Rearranging, we obtain
\begin{align*}
    \frac{1}{m+1}\sum_{i=m+1}^{n+1} f(x_i) (\alpha-\mathbf 1\{\varphi_j(x_i)>\hat h^{1-\alpha}_{\varphi_j }(x_i)\}) = \\ \frac{1}{m+1}\sum_{\varphi_j(x_i) = \hat h^{1-\alpha}_{\varphi_j }(x_i)}(\alpha-t^\alpha_i)f(x_i) -\frac{d}{d\epsilon}R(\hat h^{1-\alpha}_{\varphi_j }+\epsilon f) \big |_{\epsilon=0}.
\end{align*}
We can relate the left-hand side of the above equation to a desired coverage guarantee as follows:
\begin{align*}
&\mathbb E\!\left[
f(x_{n+1})
\Big(
\mathbf 1\{\varphi_j(x_{n+1}) \in \hat C_j(x_{n+1})\}
- (1-\alpha)
\Big)
\right] \\
&\quad =
\mathbb E\!\Big[
f(x_{n+1})
\Big(
\alpha
- \mathbf 1\{\varphi_j(x_{n+1}) > \hat h^{1-\alpha/2}_{\varphi_j}(x_{n+1})\}
- \mathbf 1\{\varphi_j(x_{n+1}) < \hat h^{\alpha/2}_{\varphi_j}(x_{n+1})\}
\Big)
\Big].
\end{align*}

Since both $\hat h^{1-\alpha/2}_{\varphi_j }$ and $\hat h^{\alpha/2}_{\varphi_j}$ are fit symmetrically, 
$\{(f(x_i), \hat h^{\alpha/2}_{\varphi_j}, \hat h^{1-\alpha/2}_{\varphi_j } (x_i), \varphi_j(x_i))\}_{i=1}^{n+1}$ are exchangeable with each other. Thus,
\begin{multline}
\label{D1}
    \mathbb E\left[f(x_{n+1})(\alpha-\mathbf 1\{\varphi_j(x_i) >\hat h^{1-\alpha/2}_{\varphi_j }(x_{n+1})\}-\mathbf 1\{\varphi_j(x_i)<\hat h^{\alpha/2}_{\varphi_j}(x_{n+1})\})\right] \\
    = \mathbb E\Big[f(x_{n+1})(\alpha/2-\mathbf 1\{\varphi_j(x_i) >\hat h^{1-\alpha/2}_{\varphi_j }(x_{n+1})\})\ \\
    +f(x_{n+1})(\alpha/2 -\mathbf 1\{\varphi_j(x_i)<\hat h^{\alpha/2}_{\varphi_j}(x_{n+1})\})\Big] \\
    = \mathbb E\Big[f(x_{n+1})(\alpha/2-\mathbf 1\{\varphi_j(x_i) >\hat h^{1-\alpha/2}_{\varphi_j }(x_{n+1})\})\\-f(x_{n+1})(1-\alpha/2 -\mathbf 1\{\varphi_j(x_i)>\hat h^{\alpha/2}_{\varphi_j}(x_{n+1})\})\Big] \\
    = \mathbb E\Big [\frac{1}{m+1}\sum_{i=m+1}^{n+1}f(x_i)(\alpha/2-\mathbf 1\{\varphi_j(x_i) >\hat h^{1-\alpha/2}_{\varphi_j }(x_i)\}) - \\
    \frac{1}{m+1}\sum_{i=m+1}^{n+1}f(x_i)(1-\alpha/2-\mathbf 1\{\varphi_j(x_i) >\hat h^{\alpha/2}_{\varphi_j }(x_i)\})\Big ] \\
    = \mathbb E\Big [\frac{1}{m+1}\sum_{i=m+1}^{n+1}(\alpha/2 - t^{\alpha/2}_i)f(x_i)\mathbf 1\{\varphi_j(x_i)= \hat h^{1-\alpha/2}_{\varphi_j }(x_i)\}\Big ]- \\\mathbb E \Big [\frac{d}{d\epsilon}R(\hat h^{1-\alpha/2}_{\varphi_j } + \epsilon f) \Big |_{\epsilon=0} \Big] -\\
    \mathbb E\Big [\frac{1}{m+1}\sum_{i=m+1}^{n+1}(1 - \alpha/2 - t^{1-\alpha/2}_i)f(x_i)\mathbf 1\{\varphi_j(x_i)= \hat h^{\alpha/2}_{\varphi_j }(x_i)\}\Big ] + \mathbb E \Big [\frac{d}{d\epsilon}R(\hat h^{\alpha/2}_{\varphi_j } + \epsilon f) \Big |_{\epsilon=0} \Big].
\end{multline}
Since $\alpha -t^{i*}_{\alpha} \in [0,1]$, for non-negative $f$, we have lower bound
$$\frac{1}{m+1}\sum_{i=m+1}^{n+1}(\alpha - t^{i*}_{\alpha})f(x_i)\mathbf 1\{\varphi_j(x_i)= \hat h^{1-\alpha}_{\varphi_j }(x_i)\}\geq0.$$
Applying this to \eqref{D1} gives us the first part
of Lemma~\ref{lem:conditional}:
\begin{align*}
   &  \mathbb E [f(x_{n+1})(\mathbf 1\{\varphi_j(x_i)\in\hat C_j(x_{n+1})\}-(1-\alpha))] \\
   & \geq - \mathbb E \Big [\frac{d}{d\epsilon}R(\hat h^{1-\alpha/2}_{\varphi_j } + \epsilon f) \Big |_{\epsilon=0} \Big] + \mathbb E \Big [\frac{d}{d\epsilon}R(\hat h^{\alpha/2}_{\varphi_j } + \epsilon f) \Big |_{\epsilon=0} \Big].
\end{align*}

Relaxing the assumption of $f$'s non-negativity and using the exchangeability property, we also have upper bound
\begin{multline*}
\Big |\mathbb E \Big [ \frac{1}{m+1}\sum_{i=m+1}^{n+1}(\alpha/2 - t^{\alpha/2}_i)f(x_i)\mathbf 1\{\varphi_j(x_i)= \hat h^{1-\alpha/2}_{\varphi_j }(x_i)\} \Big ] - \\
\mathbb E\Big [\frac{1}{m+1}\sum_{i=m+1}^{n+1}(1 - \alpha/2 - t^{1-\alpha/2}_i)f(x_i)\mathbf 1\{\varphi_j(x_i)= \hat h^{\alpha/2}_{\varphi_j }(x_i)\}\Big ] \Big | \\ 
\leq 
\Big |\mathbb E \Big [ \frac{1}{m+1}\sum_{i=m+1}^{n+1}(\alpha/2 - t^{\alpha/2}_i)f(x_i)\mathbf 1\{\varphi_j(x_i)= \hat h^{1-\alpha/2}_{\varphi_j }(x_i)\} \Big ]\Big | + \\
\Big |\mathbb E\Big [\frac{1}{m+1}\sum_{i=m+1}^{n+1}(1 - \alpha/2 - t^{1-\alpha/2}_i)f(x_i)\mathbf 1\{\varphi_j(x_i)= \hat h^{\alpha/2}_{\varphi_j }(x_i)\}\Big ] \Big | \\
\leq \mathbb E\Big [\frac{1}{m+1} \sum_{i=m+1}^{n+1}|f(x_i)|\mathbf 1\{\varphi_j(x_i)= \hat h^{1-\alpha/2}_{\varphi_j }(x_i)\} \Big ] \ + \\
\mathbb E \Big [\frac{1}{m+1} \sum_{i=m+1}^{n+1}|f(x_i)|\mathbf 1\{\varphi_j(x_i)= \hat h^{\alpha/2}_{\varphi_j }(x_i)\}\Big ] \\
= \mathbb E[|f(x_i)| \mathbf 1\{\varphi_j(x_i)= \hat h^{1-\alpha/2}_{\varphi_j }(x_i)\}] \ + 
\mathbb E[|f(x_i)| \mathbf 1\{\varphi_j(x_i)= \hat h^{\alpha/2}_{\varphi_j }(x_i)\}],
\end{multline*}
which gives us the second part of Lemma 4.1.
\end{proof}

\begin{lemma}[Interpolation Error Bound]
Assume that $(x_i, y_i),..., (x_{n+1}, Y^{n+1}) \overset{i.i.d.}\sim P$ and that $K$ is uniformly bounded. Furthermore, for $j \in \{1,...,J\}$, suppose Assumption \ref{assumption:moment} holds and that the distribution of $\varphi_j|X$ is continuous with a uniformly bounded density. Then
for any $f \in \mathcal H$ and $j \in \{1,...,J\}$,
$$\frac{|\epsilon_{int}|}{\mathbb E_P[|f(X)|]}\leq O\Big(\frac{d\log(n)}{\lambda n}\Big )\frac{\mathbb E[\max_{1\leq i \leq n+1}|f(x_i)|]}{\mathbb E_P[|f(X)|]}$$
\end{lemma}

\begin{proof}[Proof of Interpolation Error Bound]
The proof directly applies the results from \citet{gibbs2025conformal} and uses the stability of RKHS regression, but replacing the non-conformal score
with the Shapley value.

\end{proof}

\begin{lemma}[RKHS Width Bound]
\label{rkhs_width_bound}
Let $\kappa = \sup_x \sqrt{K(x,x)}$. Let $C_\Omega(x) = ||\Omega(x)||_2$ be the norm
of the finite-dimensional features at test point $x$.
We also assume that Assumption~\ref{assumption:moment} holds. Then the width $W_j(x_{n+1})$ of conditional interval $\hat C_j(x_{n+1})$ with regularization parameter $\lambda$
satisfies:
$$W_j(x_{n+1}) \leq 2\left(\kappa\sqrt{\frac B \lambda} + C_\Omega(x)\sqrt\frac{B}{\lambda_2}\right).$$
\end{lemma}

\begin{proof}[Proof of Lemma \ref{rkhs_width_bound}]
    Since $\hat h^{1-\alpha}_{\varphi_j }$ minimizes
$$J(h_{\varphi_j}^{1-\alpha})=\frac{1}{m+1} \sum_{i=m+1}^{n+1} \ell_\alpha(h^{1-\alpha}_{\varphi_j}(x_i), \varphi_j(x_i))+\lambda_1||h_K||^2_K + \lambda_2||\beta||^2_2,$$
$J(\hat h^{1-\alpha}_{\varphi_j })$ must be lower
than the objective value of the zero function ($h_{\varphi_j}^{1-\alpha} = 0$). Additionally, since
$\ell_\alpha(0, \varphi_j) \leq \max(\alpha, 1-\alpha)|\varphi_j|$, because we set both $h_k$ and $\beta$ to 0, and because we assume Assumption~\ref{assumption:moment},
$$J(0)=\frac{1}{m+1} \sum_{i=m+1}^{n+1} \ell_\alpha(h^{1-\alpha}_{\varphi_j}(x_i), \varphi_j(x_i)) + 0 \leq B.$$
Then,
$$\lambda_1||\hat h_K||^2_K \leq J(\hat h) \leq B,$$ or
$$||\hat h_K||_K \leq\sqrt{\frac{B}{\lambda_1}}.$$
Similarly,
$$\lambda_2||\hat \beta||^2_2 \leq J(\hat h) \leq B,$$ or
$$||\hat \beta||_2 \leq\sqrt{\frac{B}{\lambda_2}}.$$

We exploit the reproducing property of the RKHS,
which states that $h_K(x) = \langle h_K, K(x,\cdot)\rangle_K| \leq ||h||.$
Using the Cauchy-Schwarz inequality, we get:
$$|h_K(x)| = |\langle h_K, K(x,\cdot)\rangle_K| \leq 
||h_K||_K \cdot ||K(x, \cdot)||_K.$$
Again, from the reproducing property and Cauchy-Schwarz, $$||K(x,\cdot)||_K = \sqrt{\langle K(x,\cdot), K(x,\cdot)\rangle} = \sqrt{K(x,x)} \leq \kappa.$$
So 
\begin{align*}
    |\hat h(x)| &\leq \kappa||\hat h_K||_K + ||\hat \beta||_2 \cdot ||\Omega(x)||,\\
    |\hat h(x)| &\leq \kappa\sqrt{\frac{B}{\lambda_1}} + C_\Omega(x) \sqrt{\frac{B}{\lambda_2}}.
\end{align*}
Let $\hat h^L(x)$ be the lower bound estimate, and 
let $\hat h^U(x)$ be the upper bound estimate.

Thus, $W(x) \leq |\hat h^U(x)| + |\hat h^L(x)| \leq 2(\kappa\sqrt{\frac B \lambda_1} + C_\Omega(x)\sqrt\frac{B}{\lambda_2})$.
\end{proof}

%In the case of regression, $B = Var(Y)$, so standard target to have
%unit variance

\begin{proof}[Proof of Theorem~\ref{thm:classification}] 
Let $\hat h_j^L:=\hat h _{\varphi_j}^{\frac{\alpha}{2q}}(x_{n+1})$
and let $\hat h_j^U := \hat h _{\varphi_j}^{1-\frac{\alpha}{2q}}(x_{n+1})$.
    We define failure as the case where any
    individual modality in the selected set overestimates its value, or any modality
    in the optimal set underestimates its value.
    Let $f(\cdot) = f_K(\cdot)+\Omega(\cdot)^\top \beta$ indicate the
    weighting of interest and we deconstruct $\hat h_j^L = \hat h^L_{j,K}(\cdot)+\Omega(\cdot)^\top\hat \beta^{n+1}$.
    Applying a union bound and Lemma 2, we get:
    \begin{align*}
       & \mathbb P(\underset{j \in \hat S^\alpha_q}\bigcup\varphi_j(x_{n+1}) < \hat h_j^L \underset{j \in S^*_q}\bigcup \varphi_j(x_{n+1}) > \hat h^U_j) \\
       &\leq
    \sum_{j\in \hat S^\alpha_q} \mathbb P(\varphi_j(x_{n+1}) < \hat h_j^L) + \sum_{j\in S^*_q} \mathbb P(\varphi_j(x_{n+1}) > \hat h_j^U) \\ 
    &\leq \sum_{j\in \hat S^\alpha_q}(\frac{\alpha}{2q})+
    \sum_{j\in S^*_q}(\frac{\alpha}{2q}) + 2\lambda_1 \epsilon_1+2\lambda_2 \epsilon_2 \\
    &\leq \frac{|\hat S^\alpha_q|+|S^*_q|}{2q}\alpha + 2\lambda_1\epsilon_1 + 2\lambda_2\epsilon_2\\
    &= \alpha + 2\lambda_1\epsilon_1 + 2\lambda_2\epsilon_2 +\sum_{j \in \hat S_q^\alpha}\frac{\epsilon_{int,j}}{\mathbb E_P [f(X)]} + \sum_{j \in S^*_q}\frac{\epsilon_{int,j}}{\mathbb E_P [f(X)]},
    \end{align*}
    where $$\epsilon_1=\sum_{j\in \hat S^\alpha_q} \frac{\mathbb E[\langle \hat h_{j, K}^{L}, f_K \rangle_K]}{\mathbb E_P [f(X)]}+
    \sum_{j\in S^*_q} \frac{\mathbb E[\langle \hat h_{j, K}^{L}, f_K \rangle_K]}{\mathbb E_P [f(X)]}$$
    and
    $$\epsilon_2 = \sum_{j\in \hat S^\alpha_q} \frac{\mathbb E[\langle \hat \beta^{n+1}, \beta \rangle]}{\mathbb E_P [f(X)]}+
    \sum_{j\in S^*_q} \frac{\mathbb E[\langle \hat \beta^{n+1}, \beta \rangle]}{\mathbb E_P [f(X)]},$$
    and modality interpolation errors $\epsilon_{int,j}$.

Following \citet{he2024efficient}, we choose
loss function $\ell(y, \hat y) := -\mathbf 1(y=1) \log \hat y - \mathbf 1(y=0) \log (1-\hat y)$.

Proposition 3.2 from \citet{he2024efficient} demonstrates that the utility function is
monotonically non-decreasing, meaning that a larger set of modalities
will always have a utility function as high or higher than a smaller set.

We define the Shapley value of the utility function: 
 
\begin{align*}
    \varphi_j &= \mathbb E[\varphi_j(x_{n+1})|\hat \mu] \\&=\sum_{S\subseteq[p]\setminus\{j\}}\frac{|S|!(p-|S|-1)!}{p!} \{\mathbb E[\text{val}(X_i, Y_i, S\cup J)] - \mathbb E[\text{val}(X_i, Y_i, S)]\}
\end{align*}

Now that we have these pieces, we also apply Propositions 5.1, 5.2, and 5.3
from \citet{he2024efficient}, which under Assumption~\ref{assumption:independence}, results in the following equation 
for some set $S$ of size at most $q$:
\begin{align*}
    \sum_{j\in S} I(X_{ij},Y) - q\delta \leq \sum_{j\in S} I(X_{ij},Y) - |S|\delta \leq\sum_{j\in S} \varphi_j &\leq \sum_{j\in S} I(X_{ij},Y) + |S|\delta \leq \sum_{j\in S} I(X_{ij},Y) + q\delta \\
    I(X_{i,S},Y) - 2q\delta \leq \sum_{j\in S} \varphi_j &\leq I(X_{i,S},Y) + 2q\delta \\
    \mathbb E[\text{val}(X_i, Y_i, S)] - 2q\delta \leq \sum_{j\in S} \varphi_j &\leq \mathbb E[\text{val}(X_i, Y_i, S)] + 2q\delta
    % &= \text{Var}(\mathbb E[Y|V_{S_q^\alpha}]) \\
    % &= \text{val}(S)
\end{align*}

Since the lower bound holds for all $j \in  S^*_q$ and because 
we select a subset of size at most $q$ maximizing the sum of upper confidence bounds:
$$\mathbb E[\text{val}(X_i, Y_i, \hat S^\alpha_q)] \geq \sum_{j\in \hat S^\alpha_q} \varphi_j - 2q\delta =\sum_{j\in \hat S^\alpha_q}\mathbb E[\varphi_j(x_{n+1})]-2q\delta \geq \sum_{j\in \hat S^\alpha_q} \hat h_j^L - 2q\delta$$
$$\mathbb E[\text{val}(X_i, Y_i, S^*_q)] \leq \sum_{j\in S^*_q} \varphi_j +2q\delta =\sum_{j\in S^*_q}\mathbb E[\varphi_j(x_{n+1})]+2q\delta\leq \sum_{j\in  S^*_q} \hat h_j^U + 2q\delta \leq \sum_{j\in \hat S^\alpha_q} \hat h_j^U + 2q\delta.$$

Because $\hat h_j^L= \hat h_j^U - W_j$ for all $j \in \hat S^\alpha_q$
$$\sum_{j\in \hat S_q^{\alpha}} \hat h_j^U
=
\sum_{j\in \hat S_q^{\alpha}} (\hat h_j^L + W_j)
$$
$$\sum_{j\in \hat S_q^{\alpha}} \hat h_j^L
\ge
\mathbb E[\text{val}(X_i, Y_i, S_q^*)] - \sum_{j\in \hat S_q^{\alpha}} W_j - 2q\delta.
$$
Remember that the width $W_j$ of our interval is
bounded according to Lemma~\ref{rkhs_width_bound}.
Then,

\begin{align*}
    \mathbb E[\text{val}(X_i, Y_i, \hat S^\alpha_q)] &\geq \mathbb E[\text{val}(X_i, Y_i, S^*_q)] - \mathbb \sum_{j\in \hat S^\alpha_q} W_j - 4q\delta \\
    &= \mathbb E[\text{val}(X_i, Y_i, S^*_q)] - \sum_{j\in \hat S^\alpha_q} 2(\kappa\sqrt{\frac B \lambda_1} + C_\Omega(x)\sqrt\frac{B}{\lambda_2}) - 4q\delta
\end{align*}

\end{proof}

% \begin{theorem}[Selection Consistency for Regression]
% Let $S^*$ be the optimal set of size $q$.
% Let $S^\alpha_q$ be the set of size $q$ selected
% by our algorithm for target 
% failure probability $\alpha \in (0,1)$, and
% let $\ell(y, \hat y) = (y-\hat y)^2$
% Under Assumptions~ .........,
% the following inequality holds with probability
% at least $1 - \alpha - 2\lambda \epsilon_n$:
% $$\text {val} (S^\alpha_q) \geq \text{val}(S^*)-\sum_{j\in S^*} 2(\kappa\sqrt{\frac B \lambda} + C_\Omega(x)\sqrt\frac{B}{\lambda_2})$$

% \end{theorem}

In the case of regression, we choose loss function
$\ell(y, \hat y) := (y-\hat y)^2$.

\begin{proof}[Proof of Theorem~\ref{thm:regression}] 
    The proof follows the proof of Theorem~\ref{thm:classification}, but without a
    $\delta$ term. Here we replace
    \citet{he2024efficient}'s Proposition 3.2 with their Proposition 3.5,
    and Proposition 5.1, 5.2, and 5.3 with their
    Proposition 5.5, which allows strict additivity and monotonicity
    of Shapley values $\varphi$.
\end{proof}
\end{document}